\documentclass[11pt]{article}
\usepackage{amsmath,amssymb,amsthm,hyperref,multirow}
\usepackage{enumitem}

\usepackage{textcomp}

\usepackage{booktabs}

\usepackage{color}

\newsavebox{\measurebox}
\usepackage[numbers,sort&compress]{natbib}
\usepackage{authblk}
\usepackage{algorithm}
\usepackage{algpseudocode}
\usepackage[table]{xcolor}
\usepackage{graphicx}
\usepackage{epstopdf}
\usepackage{subfig}
\usepackage{dsfont, physics,stmaryrd,tabu}
\usepackage{enumitem}
\usepackage{mathtools}

\theoremstyle{plain}
\newtheorem{theorem}{Theorem}[section]
\newtheorem*{theorem*}{Theorem}

\newtheorem{proposition}[theorem]{Proposition}
\newtheorem{corollary}[theorem]{Corollary}
\theoremstyle{definition}
\newtheorem{example}[theorem]{Example}
\newtheorem{definition}[theorem]{Definition}

\theoremstyle{remark}

\usepackage{authblk}

\numberwithin{equation}{section}
\numberwithin{algorithm}{section}
\numberwithin{figure}{section}
\numberwithin{table}{section}
\numberwithin{theorem}{section}

\usepackage{array}
\newcolumntype{P}[1]{>{\centering\arraybackslash}p{#1}}
\title{Extending the step-size restriction for gradient descent to avoid strict saddle points}

\author[1]{Hayden Schaeffer}

\author[2]{Scott G. McCalla}

\affil[1]{ Department of Mathematical Sciences\\
  Carnegie Mellon University\\
  Pittsburgh, PA, USA  }
\affil[2]{ Department of Mathematical Sciences\\
   Montana State University\\
Bozeman, MT, USA}

\begin{document}
\date{~}
\maketitle

\begin{abstract}
We provide larger step-size restrictions for which gradient descent based algorithms (almost surely) avoid strict saddle points. In particular, consider a twice differentiable (non-convex) objective function whose gradient has Lipschitz constant $L$ and whose Hessian is well-behaved. We prove that the probability of initial conditions for gradient descent with step-size up to $2/L$ converging to a strict saddle point, given one uniformly random initialization, is zero. This extends previous results up to the sharp limit imposed by the convex case. In addition, the arguments hold in the case when a learning rate schedule is given, with either a continuous decaying rate or a piece-wise constant schedule. 
\end{abstract}

\section{Introduction}

Gradient descent based methods are among the main algorithms for optimizing models throughout machine learning. As many learning models are non-convex, their energy landscapes may consist of spurious local minima and saddles; this may lead algorithms to learn models that do not generalize well to new data \cite{keskar2016large}. In \cite{pascanu2014saddle}, it was argued  that in high-dimensional optimization, saddle points are more problematic than local minima. It is easy to construct  examples for which gradient descent converges to saddle points given certain initialization \cite{nesterov2013introductory,lee2016gradient}. However, when the step-size is sufficiently small and the saddles are \textit{strict}, i.e. the Hessian has at least one negative eigenvalue, the gradient descent method is unlikely to converge to a saddle \cite{lee2016gradient}. On the other hand, it is still possible that gradient descent will take exponential time to escape \cite{du2017gradient}.
The strict saddle condition appears in many applications, for example, orthogonal tensor decomposition \cite{ge2015escaping}, low-rank matrix recovery \cite{bhojanapalli2016global,ge2017no,ge2016matrix}, dictionary learning \cite{sun2016complete, sun2017complete}, generalized phase retrieval \cite{sun2018geometric}, and neural networks \cite{soltanolkotabi2019theoretical}. 

First-order gradient descent based methods can avoid or escape saddles when unbiased noise is added to the system. In \cite{pemantle1990nonconvergence}, the authors prove that the Robbins--Monro stochastic approximation converges to local minima in the presence of strict saddles. For objective functions with strict saddle, \cite{ge2015escaping} provided quantitative convergence rates to local minima for the noisy gradient descent method. Convergence of the normalized gradient descent with noise injection was shown in \cite{levy2016power}.

Alternatively, deterministic methods which use second-order information or trust-regions \cite{conn2000trust} can be used to avoid strict saddles. Some examples of such methods include: the modified Cholesky factorization \cite{gill1974newton}, the modified Newton's method using negative curvature \cite{more1979use}, the cubic-regularized Newton's method \cite{nesterov2006cubic}, saddle-free Newton's method for deep learning \cite{dauphin2014identifying,pascanu2014saddle}, algorithms for higher-order saddles \cite{anandkumar2016efficient}, and more recently, trust-region approaches in \cite{sun2016complete, sun2017complete,sun2018geometric}. 

One issue with `second-order' approaches is the need for higher-order information that leads to polynomial (in dimension) complexity per-iteration. For machine learning problems, which are typically of very high-dimension, this complexity can be prohibitive. Some recent approaches  \cite{reddi2017generic,royer2018complexity, carmon2018accelerated} were proposed to lower the per-iteration complexity  of second-order methods while converging to second-order stationary points (see \cite{carmon2018accelerated}). In   \cite{jin2017escape}, the authors propose a perturbed gradient descent method which converges to the second-order condition with a poly-logarithmic cost.

\textbf{Contributions of this work.}
The recent work of \cite{lee2016gradient,panageas2016gradient, lee2017first} showed that, under various conditions,  the gradient descent algorithm will avoid strict saddle points (without the need for additional hyper-parameters or higher-order information). The main technique is to show that the attracting set of a strict saddle has zero measure by invoking the stable manifold theorem applied to the discrete dynamical system generated by the gradient descent method for a $C^2$ non-convex objective function $f$ with step-size $\alpha>0$. In \cite{lee2016gradient}, they proved that gradient descent avoids strict saddles if the gradient of the objective function has Lipschitz constant $L$ (globally), isolated saddle points, and $\alpha < 1/L$. In \cite{lee2017first}, it was shown that many first-order methods will avoid strict saddles under these conditions.  Accelerated methods, such as the heavy-ball method, also avoid strict saddles as shown in \cite{o2017behavior}.

The results still hold with weaker conditions. In particular, \cite{panageas2016gradient} showed that a non-global Lipschitz constant $L$ (in a convex forward invariant set) and $\alpha < 1/L$ were enough.  If the objective function is coercive, then the sublevel sets are compact and $L$ does not have to be global; however, the results of \cite{panageas2016gradient} hold more generally. They also showed that over the set of all local minima $C$, if $$0<\gamma<\inf\limits_{x \in C} ||\nabla f(x)||_2<\infty,$$ then $\alpha<2/\gamma$ is a necessary condition for gradient descent to converge to a local minima.

There are still several open questions, in particular, if the step-size restriction $\alpha< 1/L$ is necessary for avoiding strict saddles and if varying step-sizes effects these results \cite{lee2016gradient, lee2017first}. In this work, we show that if the set of points that obtain the Lipschitz constant is measure zero, then the bound can be extended to $\alpha=1/L$. Furthermore, a step-size of $\alpha < 2/L$ is possible if $\alpha^{-1}$ is not equal to an eigenvalue of the Hessian outside of a null-set.  Examples highlight the need for such conditions. In addition, we show that these arguments can apply to gradient descent with a varying step-sizes.

\section{Overview and Examples}

To solve the non-convex optimization problem 
$$\min_{x \in \mathbb{R}^d} f(x)$$
consider the gradient descent method with fixed step-size $\alpha>0$, i.e.:
$$x^{n+1} = x^n - \alpha \, \nabla f(x^n).$$
 The sequence $x^n$ is generated by the iterative map $x^{n+1}=g(x^n)=g^n(x^0)$ where $$g(x):= x - \alpha \, \nabla f(x).$$ Given conditions on $f$ and $\alpha$, the method will converge to a critical point of $f$ (or equivalently a fixed-point of the map $g$) \cite{absil2005convergence}.

\begin{definition}\label{defSS}
Consider a function $f:\mathbb{R}^d \rightarrow \mathbb{R}$ and assume $f \in C^2(\mathbb{R}^d)$. We define the following:
\begin{itemize}
\item A point $x^* \in \mathbb{R}^d$ is a critical point of $f$ if $\nabla f(x^*)=0$.
\item A critical point $x^* \in \mathbb{R}^d$ is a saddle point of $f$ if for all neighborhoods $U(x^*)$ around $x^*$ there exists $x,y \in U(x^*)$ such that  $$f(x) \leq f(x^*) \leq f(y).$$
\item A critical point $x^* \in \mathbb{R}^d$ is a strict saddle point if there is a negative eigenvalue, i.e. $$\lambda_j\left( \nabla^2 f(x^*)\right)<0$$ for some $1\leq j\leq d$. 
\end{itemize}
\end{definition}
\noindent Based on this definition, local maxima are technically strict saddle points. Saddle points like $(0,0)$ of the objective function $x^2-y^3$ are avoided by the definition of strict saddles. 

Define $L$ as the Lipschitz constant of the gradient. If $f\in C^2(\Omega)$, then it is easy to see that 
$$L:=\sup\limits_{x\in \Omega} \, \| \nabla^2 f(x)\|_2.$$
It was shown in \cite{lee2016gradient,panageas2016gradient, lee2017first} that for $\alpha<L^{-1}$, gradient descent avoids strict saddle points. Extending this result to $\alpha\leq L^{-1}$ introduces issues even for smooth objective functions. It is possible for gradient descent to converge to strict saddles if there are non-trivial regions where $g$ degenerates (i.e. $Dg$ is non-invertible). In effect, the gradient flow funnels iterates towards the stable manifold of a strict saddle. To illustrate various issues, we present the following examples. 

\begin{example}\label{example0} (from \cite{nesterov2013introductory,lee2016gradient})
Consider the objective function $$f(x,y) = \frac{1}{2}x^2  +\frac{1}{4}y^4- \frac{1}{2}y^2$$ over $\Omega = \mathbb{R} \times \left(-\sqrt{\frac{11}{3}},\sqrt{\frac{11}{3}}\right)$, which has three critical points $(0,0)$ (strict saddle) and $(0,\pm 1$) (minima). The Hessian is given by:
 \begin{align*}
 \nabla^2 f(x,y) = 
\begin{bmatrix}
1&   0\\
    0 & 3y^2-1
\end{bmatrix},
\end{align*} 
 and achieves its maximum at $y=\pm \sqrt{\frac{11}{3}}$, i.e. $$L=\sup\limits_\Omega\, \| \nabla^2 f(x,y)\| = 10.$$ The gradient descent method with step-size $\alpha =L^{-1}= \frac{1}{10}$ is given by:
 \begin{align*}
\begin{bmatrix} x^{n+1}\\
y^{n+1}\end{bmatrix}
= 
\begin{bmatrix}
    \frac{9}{10} x^n\\
   \frac{11}{10} y^n - \frac{1}{10} (y^n)^3 
\end{bmatrix}.
\end{align*} 
The system is forward invariant over $\Omega$, with $x^n$ converging to $0$ and $y^n$ converging to $\text{sign}(y^0)$ (minima). The sequence will only converge to the strict saddle point $(0,0)$ on the line $(x,0)$, and thus has probability zero if the initial data is sampled uniformly from $\Omega$.
\end{example}

\begin{example}\label{example1}
Consider the objective function: $$f(x,y):=\frac{1}{4} y^2 - q(y) x^2$$ for some region of $\mathbb{R}^2$ containing the origin and let $q\in C^2$.  The gradient is given by:
 \begin{align*}
 \nabla f(x,y) = 
\begin{bmatrix}
    -2q(y)x\\
    \frac{1}{2}y-q'(y)x^2
\end{bmatrix}
\end{align*} 
and the Hessian is given by:
 \begin{align*}
 \nabla^2 f(x,y) = 
\begin{bmatrix}
 -2q(y) &   -2q'(y)x\\
    -2q'(y)x & \frac{1}{2}-q''(y)x^2
\end{bmatrix}.
\end{align*} 
If we define $q$ as a smooth interpolant between $1$ and $-1$ for $y\in (10,30)$, then we can show that even though the critical point at $(0,0)$ is a strict saddle and the flow is invertible near the strict saddle, regions of degeneracy away from the strict saddle can converge to the stable manifold, and thus with some non-zero probability converge to a strict saddle. 

For an explicit example, define $q$ by: 
\begin{align*}
q(y):=
\begin{cases}
1, &\text{ \ if \ } y\leq 10,\\
1-\frac{2}{1+\exp\left(\frac{40(y-20)}{(y-20)^2-100}\right)}, &\text{ \ if \ } y\in(10,30),\\
-1, &\text{ \ if \ } y\geq 30
\end{cases}
\end{align*} 

It is easy to check that the function $q \in C^2$.  In the region $y<10$, we have:
 \begin{align*}
 \nabla^2 f(x,y) = 
\begin{bmatrix}
 -2 &   0\\
    0 & \frac{1}{2}
\end{bmatrix}
\,\mbox{and in the region $y> 30$:}\,\,
 \nabla^2 f(x,y) = 
\begin{bmatrix}
 2 &   0\\
 0 & \frac{1}{2}
\end{bmatrix}.
\end{align*} 
The only critical point is at $(x,y)=(0,0)$ and it is a strict saddle. Note that the Lipschitz constant of $\nabla f$ in some bounded region around the strict saddle that contains $y\geq 30$, restricted to $x$ near  the origin, is given by $L= 2$ and is obtained for all $y\geq 30$ (a set of positive measure). Using gradient descent with $\alpha = L^{-1}=1/2$ yields:
 \begin{align*}
\begin{bmatrix} x^{n+1}\\
y^{n+1}\end{bmatrix}
= 
\begin{bmatrix}
    x^n+q(y^n)x^n\\
   y^n - \frac{1}{4}y^n+\frac{1}{2}q'(y^n)(x^n)^2
\end{bmatrix}.
\end{align*} 
For any initialization in $y\geq 30$, we have
\begin{align*}
\begin{bmatrix} x^{n+1}\\
y^{n+1}\end{bmatrix}
= 
\begin{bmatrix}
    0\\
   \frac{3}{4}y^n
\end{bmatrix}
\end{align*} 
which is within the stable manifold for $(0,0)$ (the iterates are pushed onto the stable manifold after one step). Therefore, given this choice of step-size, with non-zero probability (after restricting onto an appropriate bounded set), gradient descent will converge to a strict saddle.

\end{example}

Example~\ref{example1} shows that large regions of space can be attracted to the local stable manifold of a strict saddle. These domains act as focusing regions, in particular, subsets where the Hessian is degenerate (i.e. at least one zero eigenvalue) can cause the flow to focus a non-zero measure set onto a measure zero stable manifold. This behavior will be taken into account in Theorem~\ref{Thrm1}.

In the next section, we provide qualitative and quantitative results on the convergence of gradient descent, in particular, the divergence from strict saddles when the step-size does not degenerate the Hessian on non-null sets.


\section{Conditions for Avoiding Strict Saddles}

For convex optimization problems with Lipschitz gradients, convergence of the gradient descent method is guaranteed for step-sizes satisfying $\alpha L \leq 1$. It is possible to take larger step-sizes. For example, if $A$ is a symmetric positive definite matrix, then gradient descent with fixed step-size will converge to a minimizer of: $$f(x) = \frac{1}{2}x^T Ax-b^Tx$$ if and only if $\alpha L < 2$. Taking $\alpha L<2$ as a reasonable upper limit for the step-size, our goal is to show that with the time-s restriction and a condition on the size of the degenerate set, gradient descent will not converge to a strict saddles. Note that this does not imply convergence to a minimizer, since non-strict saddles are possible.

The behavior near a critical point can be characterized by the well-known center manifold theorem.

\begin{theorem}
(Center Manifold Theorem \cite{shub2013global}) Let $x^*$ be a fixed point of a $C^1$ local diffeomorphism $g: U \rightarrow \mathbb{R}^d$ where $U$ is a neighborhood of $x^*$ in  $\mathbb{R}^d$. Let $E^s\bigoplus E^c \bigoplus E^u$ be an invariant splitting of $\mathbb{R}^d$ into the generalized eigenspace of $Dg(x^*)$ corresponding to the eigenvalues of absolute value less than one, equal to one and greater than one. Then for each of the invariant subspaces: $E^s$, $E^s\bigoplus E^c$, $ E^c$, $ E^c \bigoplus E^u$, and $E^u$ there is an associated local $g$ invariant $C^1$ embedded disc $W_{\text{loc}}^{s}$, $W_{\text{loc}}^{cs}$, $W_{\text{loc}}^{c}$, $W_{\text{loc}}^{cu}$, and $W_{\text{loc}}^{u}$ tangent to the linear subspace at $x^*$ and a ball $B$ around $x^*$ such that there is a norm with:
\begin{itemize}
\item[(1)] $W_{\text{loc}}^{s}=\{x \in B : g^n(x) \in B$ for all $ n\geq0 $ and $ d(g^n(x),0)\rightarrow 0$ exponentially$\}$. Also, $g:W_{\text{loc}}^{s} \rightarrow W_{\text{loc}}^{s}$ is a contraction map.
\item[(2)] $g(W_{\text{loc}}^{cs}) \cap B \subset W_{\text{loc}}^{cs}$. If $g^n(x)\in B$ for all $n\geq 0$, then $x \in W_{\text{loc}}^{cs}$.
\item[(3)] $g(W_{\text{loc}}^{c}) \cap B \subset W_{\text{loc}}^{c}$. If $g^n(x)\in B$ for all $n\in\mathbb{Z}$, then $x \in W_{\text{loc}}^{c}$.
\item[(4)] $g(W_{\text{loc}}^{cu}) \cap B \subset W_{\text{loc}}^{cu}$. If $g^n(x)\in B$ for all $n\leq 0$, then $x \in W_{\text{loc}}^{cu}$.
\item[(5)] $W_{\text{loc}}^{u}=\{x \in B : g^n(x) \in B $ for all 
$ n\leq0 $ and $ d(g^n(x),0)\rightarrow 0 $ exponentially$\}$. Also, $g^{-1}:W_{\text{loc}}^{u} \rightarrow W_{\text{loc}}^{u}$ is a contraction map.
\end{itemize}
\label{thrm:cmt}
\end{theorem}
If the gradient descent method remains close to a critical point for all time, then it is on the center-stable manifold.  Note that $W_{\text{loc}}^{s}\subset W_{\text{loc}}^{cs}$.  

\begin{theorem}
Let $f$ be a $C^2(\Omega)$ function where $\Omega$ is a forward invariant convex subset of $\mathbb{R}^d$ whose gradient has Lipschitz constant $L$. Consider the gradient descent method: $g(x) = x-\alpha \, \nabla f(x)$ with $\alpha L \in (0,2)$  and assume that the set $$\left\{ x\in \Omega \ \big| \  \alpha^{-1} \in \sigma(\nabla^2 f(x)) \right\}$$ has measure zero and does not contain saddle points. Then the probability of gradient descent converging to a strict saddle, given one uniformly random initialization in $\Omega$, is zero.
\label{Thrm1}
\end{theorem}

\begin{proof}

For simplicity of exposition, all sets are assumed to be in $\Omega$, otherwise, one can either shrink the set or replace the set with the intersection with $\Omega$ (depending on the context).

 First we will show that $g^{-1}$ maps null sets to null sets (in $\Omega$), which follows from the assumption that $g$ is $C^1$ and the set $$\left\{ x\in \Omega  \ \big| \  \alpha^{-1} \in \sigma(\nabla^2 f(x)) \right\}$$ has measure zero. The map $g$ is non-invertible only on the set $$A:=\left\{ x \in \Omega \ \big| \ \text{det}(Dg(x))=0 \right\}$$ which is equivalent to:
 \begin{align}
 A&=\left\{ x\in \Omega \, \big| \,  0 \in \sigma(D g(x)) \right\}\\ 
 &=\left\{ x \in \Omega \, \big| \,  0 \in \sigma(I - \alpha \nabla^2f(x)) \right\}\nonumber\\
 &=\left\{ x \in \Omega\, \big| \,  \alpha^{-1} \in \sigma(\nabla^2 f(x)) \right\}. \nonumber
 \end{align}
 Note that if $\alpha L<1$, then this set is measure zero by definition. For a point  $x\in \Omega \setminus A$, we can find a neighborhood of $x$ such that $\text{det}(Dg(x))\neq0$ by continuity. By the inverse function theorem  $g^{-1}$ is continuous differentiable. This implies that $g$ maps sets of measure zero to sets of measure zero in $\Omega\setminus A$. To extend it to all of $\Omega$, consider the following.  Let $ \left\{ V_j \right\}_j$ be a collection of open neighborhoods that form a (countable) covering of $\Omega \setminus A$ such that $V_j\cap A=\emptyset$: construct such a covering by first finding a neighborhood for each $x\in\Omega\setminus A$ that avoids $A$, and then applying Lindel\"of's lemma to find a countable subcovering. Given an arbitrary null set $U \subset \Omega$, we have 
 $$g^{-1}(U) \subset  A \cup \left( \ \cup_{j} \ \left(V_j \cap g^{-1}(U) \right)\ \right).$$
 The inverse function theorem can then be applied to each set $V_j \cap g^{-1}(U)$, therefore since each set has measure zero then the countable union has zero measure. This implies that the set $g^{-1}(U)$ also has measure zero. Since $U$ is arbitrary, this shows that $g^{-1}$ sends null sets to null sets (within $\Omega$).

Next, we want to show all initializations that are mapped to degenerate points in $A$ form a measure zero set. The set of all points in $\Omega$ which are iteratively mapped into $A$ by $g$ is equivalent to: $$\bigcup\limits_{j=1}^{\infty} \ g^{-j}(A)$$
and has zero measure since it is the countable union of measure zero sets. By assumption, $\Omega$ is forward invariant, thus initializations in $\Omega$ cannot lead to degenerate points outside of $\Omega$. This implies that the probability of a random initialization in $\Omega$ mapping to a degenerate point is zero.

Finally, we want to show that the set of initializations that converge to a strict saddle point has zero measure. Let $$x^0 \in \Omega \setminus  \bigcup\limits_{j=1}^{\infty} \ g^{-j}(A)$$ such that $\lim g^n(x^0)$ converges to a strict saddle $x_k$. Note that along this trajectory $g^n(x^0)$ is not in $A$ and thus is non-degenerate.  Then by the inverse function theorem and the assumption, it is a local $C^1$ diffeomorphism. Since $g$ is continuously differentiable and non-degenerate at the strict saddle point $x_k$, there exists an open neighborhood $U(x_k)$ around $x_k$ such that the spectrum of $Dg(x_k)$ is non-zero, and thus $A \cap U(x_k) = \emptyset$. For each strict saddle point, there exists a ball $B(x_k) \subset U(x_k)$ that satisfies the conditions in Theorem \ref{thrm:cmt}. The collection of such balls (over all strict saddle points) $$\bigcup\limits_{k} \, B(x_k)$$  are an open cover of the strict saddle points, so there exists a countable subcover, i.e. $$\bigcup\limits_{k} \, x_k \in \bigcup\limits_{\ell=1}^\infty \, B(x_\ell).$$ Thus there exists an $N$ such that $$g^n(x^0) \in \bigcup\limits_{\ell=1}^\infty \, B(x_\ell) $$ 
for all $n\geq N$. This implies that there exists an $\ell$ such that $g^n(x^0) \in B(x_\ell)$ for all $n\geq N$, and by Theorem \ref{thrm:cmt}, $g^n(x^0) \in W^{cs}_{loc}(x_\ell)$ 
for any $n\geq N$.

We will show that the set $W^{cs}_{loc}(x_\ell)$ has measure zero. By the strict saddle condition, we have that $Dg(x) = I-\alpha \nabla^2f\left(x \right)$ has at least one eigenvalue with magnitude greater than $1$, thus the dimension of $E^u$ is at least one, therefore $\text{dim}\left(W_{\text{loc}}^{cs}(x_\ell)\right)\leq d-1$ and the Lebesgue measure of $W_{\text{loc}}^{cs}({x_\ell})$ is zero. Since $g^n(x^0) \in B(x_\ell)$ for any $n\geq N$, we have that $$g^N(x^0) \in \bigcap\limits_{j=0}^\infty \ g^{-j}(B(x_\ell)),$$ i.e. $g^N(x^0)$ is contained in the intersection of all domains which are mapped into the ball $B(x_\ell)$. The set $$\bigcap\limits_{j=0}^\infty \ g^{-j}(B(x_\ell))$$ is contained in $W^{cs}_{loc}(x_\ell)$, so it has measure zero. Since $$g^N(x^0) \in \bigcap\limits_{j=0}^\infty \ g^{-j}(B(x_\ell)),$$ we have that $$x^0 \in g^{-N}\left(\bigcap\limits_{j=0}^\infty \ g^{-j}(B(x_\ell))\right).$$ The integer $N$ depends on the initialization $x^0$ and the fixed-point $x_\ell$, thus we must consider an arbitrary $N$. In particular, the backward map $g^{-1}$ is in $C^1$, thus the measure of $$g^{-n}\left(\bigcap\limits_{j=0}^\infty \ g^{-j}(B(x_\ell))\right)$$ is zero for all $n\geq 0$.  Note that a countable union of measure zero sets are measure zero, so the set 
$$\mathcal{S}=\bigcup\limits_{\ell=0}^\infty \bigcup\limits_{n=0}^\infty  \ g^{-n}\left(\bigcap\limits_{j=0}^\infty \ g^{-j}(B(x_\ell))\right)$$
has measure zero as well. The set $\mathcal{S}$ contains all points in $$\Omega \setminus  \bigcup\limits_{j=1}^{\infty} \ g^{-j}(A)$$ which converge to strict saddles, thus the measure of all points in $\Omega$ that converge to a strict saddle is zero.

\end{proof}

As was shown in the proof, the condition that the set $\left\{ x\in \Omega \ \big| \  \alpha^{-1} \in \sigma(\nabla^2 f(x)) \right\}$ has measure zero, implies that $g^{-1}$ has the Luzin N property over sets in $\Omega$. The following is a direct result of Theorem~\ref{Thrm1} for the step-size $\alpha L=1$.

\begin{corollary}
Let $f$ be a $C^2(\Omega)$ function where $\Omega$ is a forward invariant convex subset of $\mathbb{R}^d$ whose gradient has Lipschitz constant $L$.  Consider the gradient descent method: $g(x) = x-L^{-1} \, \nabla f(x)$ and assume that the set where $\sigma(\nabla^2 f(x))$ achieves its maximum has measure zero and does not contain saddles. Then the probability of gradient descent converging to a strict saddle, given one uniformly random initialization in $\Omega$, is zero.
\end{corollary}

Example \ref{example1} shows that the measure zero assumption on the degenerate set is necessary. In addition, note that the results above do not assume that the strict saddles are isolated.

\subsection{Weaker Condition: Positive Lipschitz Restriction}\label{Sec:positiveLip}
Define $$\ell(x)  := \max\limits_{1\leq j\leq d}\ \max(\lambda_j(x),0)$$ 
(where $\lambda_j$ is an eigenvalue of the Hessian) and let $L_+$ be the Lipschitz constant of the positive part:
$$L_+= \sup\limits_{x \in \Omega} \ \ell(x).$$ Then we can show that control of  $L_+$ is sufficient for avoiding strict saddles, although it may not imply convergence to minima. 


\begin{example}\label{example2}
Consider the objective function $f(x,y):=Q(x)+\frac{1}{b} y^2 $, where $Q$ is defined as the even function with:
\begin{align*}
Q(x) = 
\begin{cases}
&a\cos(x), \quad  \text{if} \ \ x\leq \tilde{x}\\
&\frac{1}{b}\left(x-\tilde{x}-\frac{ab}{2}\sin(\tilde{x})\right)^2-\frac{2}{b}-\frac{a^2b}{4}\sin^2(\tilde{x}), \quad  \text{if}\ \  x>\tilde{x}\\
\end{cases}
\end{align*} 
and where $\tilde{x}=\arccos(-\frac{2}{ab})$ with $ab\geq2$ and $a$ and $b$ positive (thus $\tilde{x} \in [\pi/2,\pi]$).
 The function has three critical points: $(0,0)$ a strict saddle and two minima defined at $\pm (\tilde{x}+\frac{ab}{2}\sin(\tilde{x}),0)$. The Hessian is diagonal with eigenvalues given by $Q''(x)$ and $\frac{2}{b}$. The Lipschitz constant is $L=a$ and is obtained at $x=0$ and the positive Lipschitz constant is $L_+ = \frac{2}{b}$. 

Consider the gradient descent method with $\alpha = L_{+}^{-1}=\frac{b}{2}$, then $y^{n} = 0$ for all $n>1$. The iterative map for $x^n$ is define by:
\begin{align*}
x^{n+1} = 
\begin{cases}
&x^{n}+\frac{ab}{2} \sin(x^n), \quad  \text{if} \ \ |x|\leq \tilde{x}\\
&\tilde{x}+\frac{ab}{2}\sin(\tilde{x}), \quad  \text{if}\ \  x>\tilde{x}\\
&-\tilde{x}-\frac{ab}{2}\sin(\tilde{x}), \quad  \text{if}\ \  x<-\tilde{x}.
\end{cases}
\end{align*} 
For points in $0 < |x| < \tilde{x}$, the map expands away from zero (since in $|x|< \pi$, sin(x) and $x$ share the same sign). Therefore, points in $0 < |x| < \tilde{x}$ will flow to $|x|\geq\tilde{x}$. For any point $|x|\geq\tilde{x}$, the map will converge (in one-step) to $\pm (\tilde{x}+\frac{ab}{2}\sin(\tilde{x}))$. This shows that even if $L/L_+$ is arbitrary large, control of $L_+$ will be sufficient to avoid the strict saddle point.
\end{example}

Recall that $Dg(x) = I-\alpha\,\nabla^2 f(x)$, and if we assume  $\alpha L_+ <1$, then all eigenvalues of $Dg(x)$ are strictly positive. Since the spectrum of $Dg(x)$ is strictly positive and $g \in C^1$, then 
by the inverse function theorem, $g$ is a diffeomorphism under the positive Lipschitz condition. Following \cite{lee2016gradient,panageas2016gradient}, one can extend the result that the probability of converging to a strict saddle is zero. In particular, we have the following refinement.

\begin{proposition}
If $f \in C^2(\Omega)$ where $\Omega$ is a forward invariant convex subset of $\mathbb{R}^d$ whose gradient has positive Lipschitz constant $L_+$. Consider the gradient descent method: $g(x) = x-\alpha \, \nabla f(x)$ with $\alpha L_+ \in (0,1)$. Then the probability of gradient descent converging to a strict saddle, given one uniformly random initialization in $\Omega$, is zero.
\end{proposition}

To extend this result beyond $\alpha \, L_+<1$, we add the assumption from Theorem~\ref{Thrm1}.

\begin{corollary}
If $f \in C^2(\Omega)$ where $\Omega$ is a forward invariant convex subset of $\mathbb{R}^d$ whose gradient has Lipschitz constant $L_+$. Consider the gradient descent method: $g(x) = x-\alpha \, \nabla f(x)$ with $\alpha L_+ \in (0,2)$  and assume that the set $\left\{ x\in \Omega \ \big| \  \alpha^{-1} \in \sigma(\nabla^2 f(x)) \right\}$ has measure zero and does not contain saddles. Then the probability of gradient descent converging to a strict saddle, given one uniformly random initialization in $\Omega$, is zero.
\end{corollary}

\subsection{Varying Step-sizes}\label{Sec:varytime}

In some applications, the step-size of gradient descent changes between iterations.  We consider a variable step-size gradient descent method:
$$x^{n+1}= x^{n} - \alpha^n \, \nabla f(x^n).$$
where $\alpha^n>0$. By augmenting the iterative system with the step-size as an additional variable, we can apply the results of Theorem \ref{Thrm1} to show that the iterations avoid strict saddles.
\begin{corollary}
Let $f$ be a $C^2(\Omega)$ function where $\Omega$ is a forward invariant convex subset of $\mathbb{R}^d$ whose gradient has Lipschitz constant $L$. Consider the gradient descent method with varying step-size satisfying that $\alpha^{n+1}=h(\alpha^n)$, where $h\in C^1$ is a strictly decreasing contractive map over the interval $\mathcal{I}$ containing the unique fixed point $\alpha^*$. If $\alpha_0 L \in (0,2)$  and the set $$\bigcup\limits_{L^{-1} \leq \alpha \leq \alpha_0} \left\{ x \in \Omega \ | \ \alpha^{-1}  \in \sigma(\nabla^2 f(x)) \right\}$$ has measure zero and does not contain saddle points, then the probability of gradient descent converging to a strict saddle, given one uniformly random initialization in $\Omega$, is zero.
\label{corollary1}
\end{corollary}

\begin{proof}

 By augmenting the iterations with the step-size variable, the gradient descent method becomes:
 \begin{align*}
\begin{cases}
x^{n+1}&= x^{n} - \alpha^n \, \nabla f(x^n)\\
\alpha^{n+1}&= h(\alpha^n)
\end{cases}
\end{align*}
and can be analyzed via Theorem \ref{Thrm1}. The updated function $g(x,\alpha)$ is defined by $g(x,\alpha)= \left[x - \alpha \, \nabla \,f(x), \ h(\alpha) \right]^T$ and its Jacobian is given by:
 \begin{align*}
 D g(x,\alpha) = 
\begin{bmatrix}
 I - \alpha \, \nabla^2\,f(x) &   - \nabla f(x)\\
 0_{1\times n} & h'(\alpha)
\end{bmatrix}.
\end{align*} 
Since the Jacobian is ``block-upper-triangular'', its eigenvalues are the eigenvalues $ I - \alpha \, \nabla^2f(x)$ and $h'(\alpha)$. Since $h'$ is negative, the degeneracy in $g$ must come from $x$.  In addition, by the assumptions on $h$, $\alpha^n$ converges to $\alpha^*$ for any initialization of $\alpha^0$.

Define the set $\Omega_1 = \Omega \times \mathcal{I}$ and let $A \subset \Omega_1$ denote the set of points where $g$ is non-invertible, i.e.:
\begin{align}
 A&=\left\{ (x,\alpha) \in \Omega_1  \ \big| \  0 \in \sigma(D g(x)) \right\}\\
 &=\left\{ x \in  \Omega, \ \alpha \in  \mathcal{I} \ \big| \  0 \in \left\{ \sigma(I - \alpha \, \nabla^2\,f(x)), h'(\alpha-\alpha^*) \right\}  \right\} \nonumber\\
&= \left\{ x \in \Omega, \alpha \in \mathcal{I} \ \big| \  \alpha^{-1} \in \sigma(\nabla^2 f(x)) \right\}\\
 &= \bigcup\limits_{L^{-1} \leq \alpha \leq \alpha_0} \left\{ x \in \Omega \ | \ \alpha^{-1}  \in \sigma(\nabla^2 f(x)) \right\} \nonumber
 \end{align}
By assumption, $A$ has measure zero.

The set $\Omega_1$ is a convex subset of $\mathbb{R}^{d+1}$. By assumption, the function $g_1(x,\alpha)=x - \alpha \, \nabla f(x)$ is forward invariant on $\Omega_1$. In addition, $g_2(x,\alpha)=h(\alpha)$ is a contractive map ($|h'(\alpha)|<1$), thus $h(\mathcal{I}_1) \subset \mathcal{I}$. Therefore, $g$ is forward invariant on $\Omega_1$.

Let $$(x^0,\alpha^0) \in \Omega_1 \setminus  \bigcup\limits_{j=1}^{\infty} \ g^{-j}(A)$$ such that $\lim g^n(x^0,\alpha^0)$ converges to a strict saddle $(x,\alpha^*)$ (the fixed-point for $\alpha$ is unique). The map $g$ is continuously differentiable and non-degenerate at  $(x,\alpha^*)$, thus there exists an open neighborhood around $(x,\alpha^*)$ characterized by the product space of an open neighborhood $U(x)$ around $x$  and an open interval $S(\alpha^*)$ (which holds by the odd extension of $h$), where the spectrum of $Dg(x)$ is non-zero, thus $A \cap U(x) = \emptyset$. The rest follows from Theorem \ref{Thrm1}.

\end{proof}

The theorem above holds (trivially) if $\alpha_0<L$. If the set of step-sizes is discrete, we can simplify the results.

\begin{corollary}
Let $f$ be a $C^2(\Omega)$ function where $\Omega$ is a forward invariant convex subset of $\mathbb{R}^d$ whose gradient has Lipschitz constant $L$. Consider the gradient descent method with a finite staircase of decreasing step-sizes, i.e. $\alpha^n$ is a piecewise constant function of $n$ with finitely many jumps. If $\alpha^n L \in (0,2)$ for all $n$ and the set $\left\{ x\in \Omega \ \big| \alpha^{-1} \in \sigma(\nabla^2 f(x)) \right\}$ has measure zero for each $\alpha^n$  and does not contain saddle points, then the probability of gradient descent converging to a strict saddle, given one uniformly random initialization in $\Omega$, is zero.
\label{corollary2}
\end{corollary}

\begin{proof}
 Consider the case, $\alpha^n=\alpha_1$ for $n\leq N_1$ and $\alpha^n=\alpha_2$ for $n>N_1$. Let $g_i$ be the gradient descent method with step-size $\alpha_i$, $i=1,2$. 
 
The maps $g_i$ are $C^1$ and are non-invertible only on the set $A_i:=\left\{ x \in \Omega \ \big| \ \text{det}(Dg_i(x))=0 \right\}$ (respectively) which is equivalent to: $$A_i=\left\{ x \in \Omega \ \big| \  \alpha_i^{-1} \in \sigma(\nabla^2 f(x)) \right\}.$$ 
Following the proof of Theorem \ref{Thrm1}, $g_i^{-1}$ maps null sets to null sets (within $\Omega)$. Consider the set $A=\cup_i A_i$, which is a null set since it is a finite union of null sets. The set of all points in $\Omega$ that are  mapped to $A$ by $g_i$ (for any $i$) is equivalent to the set $$Q=\bigcup\limits_i \, \bigcup\limits_{j=1}^{\infty} \ g_i^{-j}(A).$$ Each $A_i$ is a null set, so each $g_i^{-j}(A)$ is a null set. The set $Q$ is a countable union of null sets, thus $Q$ has measure zero.

Let $x^0 \in \Omega \setminus  Q$ such that the two-step staircase gradient descent method converges to a strict saddle $x$. This can occur by two distinct scenarios : (i) $g^n_1(x^0)$ converges to $x$ within $N_1$ steps or (ii) $g^{n-N_1}_2(g^N_1(x^0))$ converges to $x$ with $n>N_1$. For case (i), using the proof of Theorem \ref{Thrm1} the set of points in $\Omega \setminus Q$ which converge to a strict saddle under $g_1$ is measure zero.

For case (ii), by assumption $x_0 \not\in Q$ so $x^{N_1}:=g_1^{N_1}(x^0) \not \in Q$, i.e. along the trajectory $g_2^{n-N_1}(x^{N_1})$ for $n>N_1$,  $g_2$  is non-degenerate and a local $C^1$ diffeomorphism.  

As before, we can show that there exists a (sufficiently large) $N$ such that $$g_2^{n}(x^{N_1})  = \bigcup\limits_{\ell=1}^\infty \, B(x_\ell)$$ for all $n\geq N$ and thus there is an $\ell$ such that $g_2^n(x^{N_1}) \in B(x_\ell)$ for all $n\geq N$ and $g_2^n(x^{N_1})\in W^{cs}_{loc}(x_\ell)$ for any $n\geq N$. This also implies that $$g_2^N(x^{N_1}) \in \bigcap\limits_{j=0}^\infty \ g_2^{-j}(B(x_\ell))$$ which is measure zero since it is contained in $W^{cs}_{loc}(x_\ell)$. Since  $g_2^N(x^{N_1}) \in \bigcap\limits_{j=0}^\infty \ g^{-j}(B(x_\ell))$, we can show that 
\begin{align*}
x^{N_1} &\in g_2^{-N}\left(\bigcap\limits_{j=0}^\infty \ g_2^{-j}(B(x_\ell))\right)\\
x^{0} &\in g_1^{-N_1}\left(g_2^{-N}\left(\bigcap\limits_{j=0}^\infty \ g_2^{-j}(B(x_\ell))\right)\right).\end{align*}
The set
$$\mathcal{S}=\bigcup\limits_{\ell=0}^\infty \bigcup\limits_{n=0}^\infty  \ g_1^{-N_1}\left(g_2^{-n}\left(\bigcap\limits_{j=0}^\infty \ g_2^{-j}(B(x_\ell))\right)\right)$$
contains all points in $\Omega \setminus Q$ which converge to strict saddles after $N_1$ iterations. The set $\mathcal{S}$ has zero measure, since each $g_i^{-1}$ maps null sets to null sets  and $\mathcal{S}$ is the countable union of null sets. Therefore the probability of case (ii) occurring is zero.

This can be generalized to finitely many discrete step-sizes, since the arguments related to the invertibility of all $g_i$ continue to hold for countable unions of null sets.

\end{proof}

\section{Discussion}
\label{conclusion}

We present several theoretical results on the conditions which guarantee that the gradient descent method will avoid a strict saddle. The results utilize the center manifold theorem, to establish the size of the attracting sets, and measure theoretic arguments, to show that the iterative maps satisfy the Luzin N condition. Our results answer an open question about the step-size posed in \cite{lee2016gradient, lee2017first}, namely, that previous claims hold for $\alpha<2L^{-1}$ with the additional assumption that the iterative map does not degenerate on non-null sets. We show that without the additional assumption, one can construct counter-examples. These results also hold for the gradient descent method with (fixed) learning rate schedules.\\

\textbf{Extensions and Applications}: The theoretical results here extend readily to other first-order methods, for example, the proximal gradient descent, block coordinate descent, etc. \cite{lee2017first}. Although the results are for uniformly random initial data, they can be easily extended to other random sampling measures. Additionally,  using the Lojasiewicz gradient inequality \cite{absil2005convergence}, one may be able to prove that if the set of critical points only contains local minima and strict saddles, then the gradient descent method converges to local minima with the extended step-sizes \cite{lee2016gradient}.\\

\textbf{Limitations}: This paper does not directly address the convergence of gradient descent to global minima or the behavior near local minima. In particular, the step-size bounds presented here may be too large for convergence when applied to a particular model. Additionally, it was shown in \cite{du2017gradient}, that the gradient descent method can take exponential time to escape a saddle, but the likelihood or predictability of such phenomena for a particular model or application is an open question. Lastly, our results on varying step-sizes utilized a fixed learning rate schedule. A line search or adaptive time-stepping method may be able to avoid saddles with weaker restrictions on $\alpha$.

\section{Acknowledgments}
H.S. would like to acknowledge the support of AFOSR, FA9550-17-1-0125 and the support of NSF CAREER grant $\#1752116$.  S.G.M. would like to acknowledge the support of NSF grant $\#1813654$.

\bibliographystyle{plain}
\bibliography{arxiv_2019.bib}

\end{document}